\pdfoutput=1
\documentclass[conference]{IEEEtran}
\IEEEoverridecommandlockouts

\usepackage[numbers,sort&compress]{natbib}   
\usepackage{amsmath,amssymb,amsfonts}
\usepackage{graphicx}
\usepackage{booktabs}
\usepackage[font=footnotesize]{caption}
\usepackage{xcolor}
\usepackage{hyperref}
\hypersetup{pdftitle={MIND: Instruction-Conditioned Incremental Grounding for Robots},colorlinks=true,linkcolor=[rgb]{0,0,0.55},citecolor=[rgb]{0,0.35,0},urlcolor=[rgb]{0,0,0.55}}
\usepackage{url}
\usepackage{tikz}
\usepackage{pifont}
\usetikzlibrary{arrows.meta,positioning,fit,backgrounds}

\newcommand{\Shat}{\hat{S}}
\newcommand{\NONE}{\textsc{none}}

\begin{document}

\title{MIND: Instruction-Conditioned Incremental Grounding for Robots}

\author{Mohamed Abouagour$^{1}$, Junhyun Nam$^{2}$, Wonse Jo$^{2}$, and Byung-Cheol Min$^{1}$%
\thanks{$^{1}$Indiana University Bloomington, USA. $^{2}$Incheon National University, Republic of Korea.}}

\maketitle

\begin{abstract}
Robots following persistent language instructions must decide which cached room judgments remain valid as their environment changes. We present MIND (Maintenance via INstruction-conditioned Dependencies), a framework for maintaining target-room decisions without re-evaluating every room after each event. A language model compiles each instruction into dependencies for semantic qualification and target selection. The controller intersects semantic dependencies with event footprints, refreshes affected judgments, and updates reachability and selection separately. For fixed deterministic evaluation, planning, and selection, conservative dependency and footprint coverage preserves the decisions of full recomputation. On $240$ five-event ProcTHOR sequences with an exact benchmark predicate, MIND reduces mean semantic evaluations from $30.6$ to $11.2$ while preserving all $1{,}200$ post-event decisions. Frozen-evaluator experiments with Phi-4 and Qwen3-14B also preserve all $1{,}200$ decisions, without improving the evaluators' underlying task accuracy. In $32$ closed-loop AI2-THOR episodes using exact-predicate qualification and metadata-derived footprints, MIND matches the final target decision of full recomputation in every episode. Unitree G1 experiments demonstrate target reuse and reselection across room transitions without additional semantic evaluations under MIND, and abstention after a barrier and an object removal at $4$ evaluations against $15$ to $18$. Additional materials are available at \url{https://mind-icra27.github.io/}.
\end{abstract}

\section{Introduction}

A robot may continue following the same instruction while its environment changes. For \textit{``bring me the nearest mug,''} entering another room, closing a doorway, or removing an object can alter the appropriate target in different ways. The robot has already evaluated which rooms satisfy the instruction, yet those cached judgments need not all become invalid together. Re-evaluating every room discards reusable computation; retaining every judgment can preserve a stale target. Each room judgment is a language-model call, which on a robot shares the GPU with perception and on a hosted model is billed, so the calls avoided are the resource at stake. We ask: \emph{after a world change, which room judgments must the robot check again?}

The event alone does not determine the required update. For \emph{the nearest chair}, robot motion can change the selected room without changing which rooms contain chairs. For \emph{a chair in the room I am in}, the same motion changes which rooms qualify. The same variable thus enters qualification or selection depending on the instruction, and its stage decides whether semantic re-evaluation is needed.

\begin{figure}[t]
\centering

\includegraphics[width=\columnwidth]{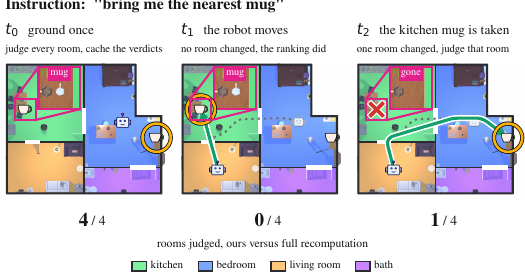}
\vspace{-16pt}
\caption{\textbf{An overview of ``Bring me the nearest mug.''} At $t_0$ every room is judged and cached. Moving
($t_1$) changes no room: judgments are reused, only the target is re-chosen. Removing the kitchen
mug ($t_2$) refreshes one judgment. Full recomputation rejudges all four each time.}
\label{fig:teaser}
\vspace{-13pt}
\end{figure}

Fig.~\ref{fig:teaser} illustrates the distinction within one persistent instruction. Moving to another room changes the ranking of candidates but not their inventory-based qualification, so no room judgment needs refreshing. Removing the kitchen mug instead requires rechecking the kitchen while retaining the other judgments.

Dependency-based reuse is well established in incremental computation, where dependencies may be declared in program or query structure or discovered during execution~\cite{acar2006selfadjusting,gupta1995maintenance,mokhov2018build,doyle1979tms}. In our setting, the state dependencies relevant to maintenance are implicit in a language instruction and differ between qualification and selection. MIND infers these instruction-specific dependencies from language, separately for qualification and selection, to determine which cached judgments require updating. We therefore study decision-preserving maintenance of a given grounding pipeline, rather than replacing its semantic evaluator or replanning the entire task after every event.

We present MIND (Maintenance via INstruction-conditioned Dependencies), a framework for \emph{instruction-conditioned incremental grounding}. Before execution, a language model compiles the instruction into separate dependency sets for room qualification and target selection. After an event, the controller intersects the semantic dependencies with a footprint of reported state changes to identify which cached judgments to refresh. Reachability and deterministic selection are maintained separately, allowing a target to change without unnecessary semantic evaluation. Under the fixed-evaluator and conservative-coverage assumptions of Section~\ref{sec:method}, MIND preserves the decisions of full recomputation. Over-approximation may increase work, whereas missing coverage can leave stale judgments; the guarantee does not certify the compiler's predictions or the evaluator's accuracy. We contribute:

\begin{itemize}
\item Instruction-conditioned, stage-specific dependency signatures combined with region-scoped event footprints, separating semantic refresh, reselection, and reuse.
\item Sufficient conditions under which this maintenance preserves full-recomputation decisions, with semantic cost as initialization plus later regional refreshes.
\item On $240$ five-event ProcTHOR sequences, mean semantic evaluations fall from $30.6$ to $11.2$ with all $1{,}200$ decisions unchanged; compiler, footprint, and interface tests examine failures outside the guarantee, $32$ closed-loop AI2-THOR episodes test maintenance during execution, and $26$ Unitree G1 runs over four scenarios exercise reuse, reselection, global refresh, and abstention on hardware.
\end{itemize}

\section{Related Work}

\noindent\textbf{Grounded planning and dynamic revision.}
Language-conditioned robots ground instructions in affordances and scene graphs~\cite{ahn2022saycan,huang2023groundeddecoding,lin2023text2motion}, revise plans from feedback~\cite{huang2022innermonologue,yao2023react,sayplan2023}, and abstain or relax infeasible goals~\cite{ren2023knowno,park2024clara,musumeci2025contextmatters}. Other systems maintain contextual or task memory for replanning~\cite{yang2026epog,kim2023dynacon,lei2025stma,xu2026mimir}; Event-Grounding Graph represents spatial structure and events~\cite{nguyen2026egg}, while LINGO-Space updates location distributions as referring expressions arrive~\cite{kim2024lingospace}. These approaches motivate maintaining useful world knowledge, but stored information can become stale~\cite{sun2026memorylies}. MIND addresses the validity of cached room-qualification judgments under a persistent instruction: which judgments a world change requires the controller to refresh, and when selection alone must be reconsidered.

\noindent\textbf{Incremental computation and embodied reuse.}
Self-adjusting computation, view maintenance, build systems, and truth maintenance reuse outputs through dependencies on changed inputs~\cite{acar2006selfadjusting,gupta1995maintenance,mokhov2018build,doyle1979tms}. Dependencies may be declared in formal descriptions or tracked during execution. KEEP manages model-state reuse, whereas AgenticCache caches and validates planning transitions~\cite{yang2026keep,agenticcache2026}. MIND shares the dependency-based reuse principle but caches regional semantic verdicts rather than model states or plan transitions: instruction-conditioned signatures separate qualification inputs from selection inputs, regional footprints scope the invalidation, and the rule preserves a fixed pipeline's decisions when both conservatively cover the required updates.

\noindent\textbf{Dependencies derived from language.}
LiP-LLM, DeComplex, and DART-LLM infer precedence among subtasks~\cite{obata2024liplm,pramanick2020decomplex,wang2024dartllm}; incremental parsing repairs grammar or lexical knowledge after parsing or planning failures~\cite{abrams2025incremental}. LG-TAP constructs temporally compact world models through backward observation search and uses efficient graph updates to reuse grounding computations by selectively recomputing world-dependent factor features~\cite{patki2023temporally}. MIND applies dependency-based reuse to a different update unit: cached room-qualification verdicts supporting a persistent target decision. It infers separate qualification and selection dependencies and combines them with event footprints, yielding explicit conditions for decision-preserving regional updates under conservative estimates.

\begin{figure*}[t]
\centering
\includegraphics[width=1\textwidth]{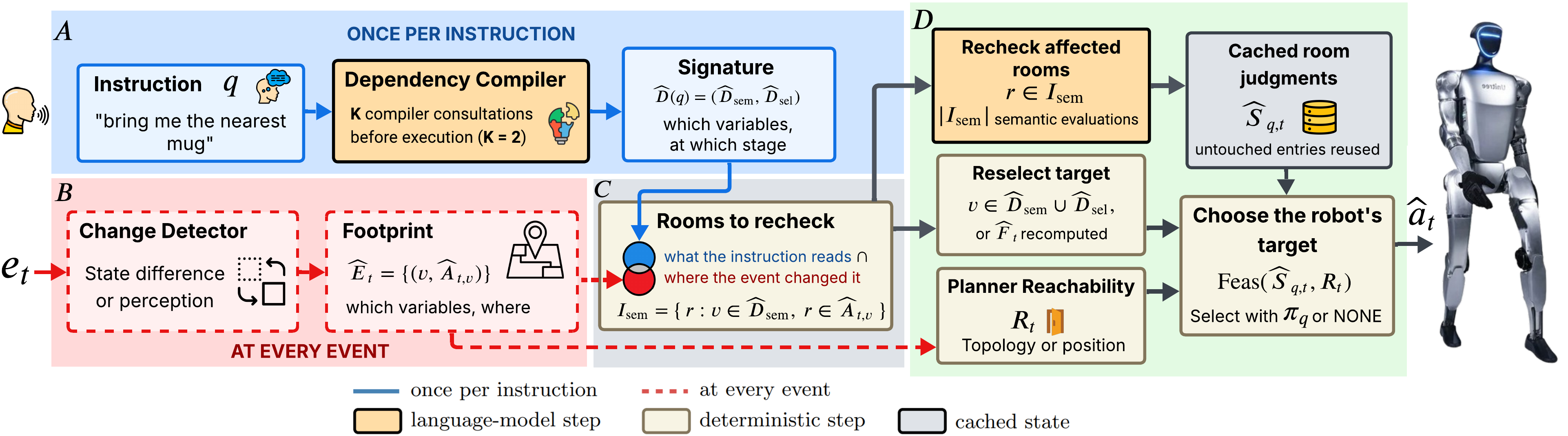}
\vspace{-10pt}
\caption{\textbf{Architecture for ``bring me the nearest mug.''} (A) Once per instruction, the compiler separates
variables needed to judge each room from those needed only to choose among qualifying rooms. (B) After
each event, the detector reports what changed and which rooms may be affected. (C) Their intersection
fixes which room judgments to recheck; position and topology changes also update reachability.
(D) The controller reuses untouched judgments and reselects a reachable target whenever the feasible set is recomputed or a reported change intersects either compiled dependency slot.}
\label{fig:arch}
\vspace{-12pt}
\end{figure*}

\section{Problem Formulation}
\label{sec:problem}

\noindent\textbf{System and event stream.}
An episode fixes a natural-language instruction $q$ and a building with regions $\mathcal{R}$; in our
indoor experiments each region is one room. The episode then evolves through events $e_{1:T}$. Event $e_t$ changes the world from
$X_{t-1}$ to $X_t$, where $X_t$ holds
region inventories, semantic region labels, the robot's region, interaction history, and the
region-adjacency graph. We denote these state categories by
\begin{equation}
\mathcal{V}=\{\text{inventory},\text{region label},\text{position},
\text{history},\text{topology}\}.
\end{equation}
We instantiate $\mathcal V$ with five state categories common to indoor grounding. Concretely,
inventory is the multiset of object types present in a region, so counts matter and removing one of several chairs is a change, region label is its semantic name,
position is the region the robot currently occupies rather than a metric pose, history is the set of regions already used to satisfy the instruction, and topology is the region-adjacency graph with per-doorway passability. History is therefore region-level: a request such as \emph{another chair} asks for a chair in a region not yet used, not for a different chair instance within a used region, which the state does not represent.
Grounding remains region-level and does not represent object coordinates or instance identities. In the graph benchmarks, \emph{nearest} means fewest doorway hops in $G_t$; the closed-loop AI2-THOR experiment instead uses navigation-grid path length. The rule requires only a variable vocabulary shared by compiler, evaluator, and detector; further categories can be added without changing it.
We assume their current values are available to the controller; perception is outside the present
formulation. The system contains a semantic evaluator, a reachability planner, a deterministic
selector with fixed tie-breaking, and a controller that maintains their results across events.

A fresh semantic evaluation, the room-level semantic check, asks whether region $r$ satisfies $q$ in
the current state,
\begin{equation}
s_q(r;X_t)\in\{0,1\}.
\end{equation}
Let $G_t$ denote the region-adjacency graph at time $t$, let $c_t\in\mathcal{R}$ be the robot's current region, and let $T$ be the number of events. The planner returns the reachable-region set $R_t=\operatorname{Reach}(G_t,c_t)$. For any binary regional
support vector $z$, define
\begin{equation}
\operatorname{Feas}(z,R_t):=\{r\in\mathcal{R}:z(r)=1\}\cap R_t.
\label{eq:compose}
\end{equation}
Full recomputation evaluates every region and returns the feasible set and the selected target
region, or \textsc{none},
\begin{equation}
\begin{aligned}
F_t^{\mathrm{full}}(q)&=\operatorname{Feas}\bigl(s_q(\cdot;X_t),R_t\bigr),\\
a_t^{\mathrm{full}}&=\pi_q\bigl(F_t^{\mathrm{full}}(q);X_t\bigr),
\end{aligned}
\label{eq:full}
\end{equation}
where the request-specific selector $\pi_q$ may rank feasible candidates using current state and
returns \NONE{} when none remain.

\noindent\textbf{Cache of room judgments and objective.}
The incremental controller stores $\Shat_{q,t}(r)$, its most recent semantic verdict for region
$r$, initialized exactly as $\Shat_{q,0}(r)=s_q(r;X_0)$ for every region. After event $e_t$, a
maintenance policy chooses a refresh set $U_t(q)\subseteq\mathcal{R}$ and updates
\begin{equation}
\Shat_{q,t}(r)=
\begin{cases}
s_q(r;X_t),&r\in U_t(q),\\
\Shat_{q,t-1}(r),&r\notin U_t(q).
\end{cases}
\label{eq:cache}
\end{equation}
It maintains $\widehat F_t(q)=\operatorname{Feas}(\Shat_{q,t},R_t)$ and reruns $\pi_q$ whenever the feasible set or a selection input may have changed; $\widehat a_t$ is the resulting decision. For $n=|\mathcal{R}|$, the semantic-evaluation cost is
\begin{equation}
C_{\mathrm{sem}}=n+\sum_{t=1}^{T}|U_t(q)|,
\label{eq:cost}
\end{equation}
against $n(T+1)$ for full recomputation. The problem is
\begin{equation}
\min C_{\mathrm{sem}}
\quad\text{subject to}\quad
\widehat a_t=a_t^{\mathrm{full}}\ \text{for every }t.
\label{eq:objective}
\end{equation}
Planner and deterministic-selection work is accounted for separately.

\noindent\textbf{Scope.}
Semantic qualification is unary in the candidate region, and its inventory dependence is restricted to that region's contents. Relations to a distinguished state-dependent region, such as the robot's current region, remain unary in the candidate, but predicates that inspect other regions' inventories are outside the locality assumption used below. Predicates over multiple free candidate regions, cross-candidate cardinality constraints, and future-state prediction are outside scope.

\section{Methodology}
\label{sec:method}
MIND compiles dependencies once per instruction and uses event footprints to maintain cached room judgments, reachability, and target selection (see Fig.~\ref{fig:arch}).

\subsection{Instruction dependency map}
The compiler (Fig.~\ref{fig:arch}) maps $q$ to
\begin{equation}
\widehat D(q)=\bigl(\widehat D_{\mathrm{sem}}(q),
\widehat D_{\mathrm{sel}}(q)\bigr),
\label{eq:signature}
\end{equation}
The true read sets are defined counterfactually for the fixed semantic function $s_q$ and selector $\pi_q$ used in the comparison. A variable $v$ belongs to $D^*_{\mathrm{sem}}(q)$ if changing only state category $v$ can change $s_q(r;X)$ for some candidate region. It belongs to $D^*_{\mathrm{sel}}(q)$ if, with the feasible set fixed, changing only $v$ can change the selected region. The two sets need not be disjoint: a variable may affect both qualification and selection. A selection-only change requires deterministic reselection but no fresh semantic verdict. These true read sets are unavailable to the controller. Reachability is maintained separately, is refreshed whenever topology or robot position changes, and is always applied through \eqref{eq:compose}.

The compiler prompt defines the five categories, explains the two slots, and supplies one development example per request family. It emits two variable lists, parsed into
\eqref{eq:signature}, with inventory included in the semantic slot by construction. Decoding is
greedy. The frozen downstream configuration compiles under $K$ cyclic orders of the six examples and
unions each slot, with $K{=}2$ unless stated; these are $K$ pre-execution consultations per instruction, never repeated after world events.

Two repairs expand a signature monotonically. \emph{Sample union} unions each slot over the $K$ orders. \emph{Stage widening} treats reported variables as possibly belonging to either stage. With $W=\widehat D_{\mathrm{sem}}\cup\widehat D_{\mathrm{sel}}$, define
\begin{equation}
\operatorname{widen}(\widehat D_{\mathrm{sem}},\widehat D_{\mathrm{sel}})
=\bigl(W,\,W\setminus\{\text{inventory}\}\bigr).
\label{eq:widen}
\end{equation}
The selector never reads inventory. Widening repairs semantic-to-selection demotions and preserves \eqref{eq:cons} whenever the unwidened signature satisfies it, at the cost of additional semantic evaluations. When widening is evaluated, it is applied after sample union. Unless explicitly labeled as widened, the main configuration uses slot-wise union only ($K{=}2$); widened configurations are reported separately. Neither operation recovers a variable absent from every compiler output.

\subsection{Event footprint}
The detector represents an observed event by
\begin{equation}
\widehat E_t=\{(v,\widehat A_{t,v})\},
\label{eq:footprint}
\end{equation}
where $v\in\mathcal V$ is a reported changed variable and $\widehat A_{t,v}\subseteq\mathcal R$ is its estimated invalidation scope. Reference footprints use changed rooms for inventory: one for a removal, source and destination for a relocation. Position, topology, history, and region label use $\mathcal R$ whenever they change, conservatively allowing nonlocal effects on qualification. This is an invalidation convention, not a claim that every room changes physically. Detector estimates may miss or add variables or regions; Proposition~\ref{prop:equiv} needs coverage of the reference footprint.

\subsection{Scoped update and robot target selection}
Intersecting
\eqref{eq:signature} and \eqref{eq:footprint} yields
\begin{equation}
\begin{aligned}
I_{\mathrm{sem}}(e_t,q)=\bigl\{r\in\mathcal{R}:&\exists(v,\widehat A_{t,v})\in\widehat E_t,\\[-2pt]
&v\in\widehat D_{\mathrm{sem}}(q),\ r\in\widehat A_{t,v}\bigr\}.
\end{aligned}
\label{eq:inval}
\end{equation}
The policy sets $U_t(q)=I_{\mathrm{sem}}(e_t,q)$ and then:
\begin{enumerate}
\item refreshes planner outputs if topology or robot position was reported changed;
\item updates \eqref{eq:cache} only for $r\in U_t(q)$;
\item recomputes $\widehat F_t(q)$ if reachability changed or $U_t(q)\neq\emptyset$; and
\item reruns $\pi_q$ if $\widehat F_t(q)$ was recomputed or a reported changed variable lies in
$\widehat D_{\mathrm{sem}}(q)\cup\widehat D_{\mathrm{sel}}(q)$.
\end{enumerate}
Only step 2 invokes the semantic evaluator. The downstream steps trigger on a report rather than
only on a changed verdict: a selection dependency that the compiler promoted into the semantic
slot may refresh no region and reproduce every support value, and step 4 still reruns selection
because the variable was reported changed.

\subsection{Conservative estimates and correctness}
The controller receives $\widehat D(q)$ and $\widehat E_t$, and either may be inexact. We say
$\widehat D(q)$ \emph{conservatively contains} $D^*(q)$ when
\begin{equation}
D^{*}_{\mathrm{sem}} \subseteq \widehat{D}_{\mathrm{sem}}
\quad\text{and}\quad
D^{*}_{\mathrm{sel}} \subseteq \widehat{D}_{\mathrm{sem}} \cup \widehat{D}_{\mathrm{sel}} .
\label{eq:cons}
\end{equation}
Let $E_t^*$ contain all actually changed categories, using the preceding footprint convention: changed rooms for inventory and $\mathcal R$ for the other categories. Define $\widehat E_t\succeq E_t^*$ when every $(v,A^*_{t,v})\in E_t^*$ has a corresponding $(v,\widehat A_{t,v})\in\widehat E_t$ with $A^*_{t,v}\subseteq\widehat A_{t,v}$.

With the semantic evaluator and its input interface fixed, approximation errors are asymmetric. Promoting a selection dependency into the semantic stage causes unnecessary re-evaluation but still triggers reselection; demoting or omitting one can leave stale support and voids the guarantee, though it need not change the target on every event. Adding dependencies or regions preserves deterministic outputs.

\newtheorem{proposition}{Proposition}
\begin{proposition}[Decision equivalence under dual conservative containment]
\label{prop:equiv}
Assume $\Shat_{q,0}=s_q(\cdot;X_0)$; all task-relevant changes are in $\mathcal V$; inventory dependencies are local to the candidate region; the semantic function $s_q$ is a fixed function of the state, chosen independently of $\widehat D$ (the reference-field interface of Section~\ref{sec:benchmarks}); \eqref{eq:cons} holds; every observed footprint satisfies $\widehat E_t\succeq E_t^*$ under the stated footprint convention; and semantic evaluation, planning, and selection are deterministic. Then $\widehat a_t=a_t^{\mathrm{full}}$ after every event.
\end{proposition}

\noindent\emph{Proof sketch.} Induct on events. The caches agree initially. If a fresh verdict
$s_q(r;X_t)$ can differ from $\Shat_{q,t-1}(r)$, some variable in
$D^*_{\mathrm{sem}}$ changed with $r$ in its true footprint. Signature and footprint containment
place $r$ in \eqref{eq:inval}, so it is refreshed; thus no stale semantic value survives and
$\Shat_{q,t}=s_q(\cdot;X_t)$. Observed topology and position changes also refresh planner output,
so $\widehat F_t=F_t^{\mathrm{full}}$. Finally, any changed variable in
$D^*_{\mathrm{sel}}$ lies in $\widehat D_{\mathrm{sem}}\cup\widehat D_{\mathrm{sel}}$, so step 4
reruns selection whenever it is reported changed, whether or not its footprint refreshed a
region. Extra reported dependencies or regions add work but cannot alter a deterministic result.
$\square$

Proposition~\ref{prop:equiv} preserves a fixed evaluator's decisions, not their task correctness. Compiled prompt fields change the evaluator itself, so agreement within that interface need not imply agreement with the reference-field evaluator; Section~\ref{sec:benchmarks} defines both interfaces and Table~\ref{tab:negctl} reports task accuracy separately.

\noindent\textbf{Semantic-evaluation cost.}
For $m_t=|U_t(q)|$, the cost is $n+\sum_t m_t$ in \eqref{eq:cost}; savings vanish when global refreshes give $m_t\approx n$. Maintaining all cached verdicts is stronger than preserving only the target, so this is not a minimum-cost solution to \eqref{eq:objective}.

\section{Benchmarks and Experimental Setup}
\label{sec:benchmarks}

The six-family dependency
and request benchmarks use automatically checked manifests frozen by hash before model evaluation. A dependency manifest is refused if a sequence has no verdict-changing event or abstains throughout, a history sequence is not moved by its use event, a move or door event changes a verdict for a nearest request, or no event exposes the inventory-only rule; a request manifest is refused if families are unbalanced or share a signature, a minimal pair adds an edge other than the one claimed, a held-out template leaks into development, or a signature disagrees with the event instrument. These are consistency checks, not independent certification of natural-language semantics. Generators, prompts, and model configurations are in the repository\footnote{\url{https://anonymous.4open.science/r/mind-icra27}}.

\noindent\textbf{Dependency benchmark.}
The primary benchmark contains $240$ five-event sequences over houses from the ProcTHOR~\cite{procthor2022} \emph{test} split, using $11$ distinct instructions across six request families. \emph{Inventory-only} reads inventory only; \emph{named-region} adds a semantic region label; \emph{current-region} adds robot position to qualification; \emph{history-dependent} adds the set of previously used regions; \emph{nearest-candidate} places position and topology in selection; and \emph{nearest+history} combines the last two. Events shut doorways, empty or relocate objects, move the robot, and consume a referent. After each event, the manifest records the reference target and the regional verdicts that actually change value. The latter define invalidation recall and precision, distinguishing missed changed verdicts from conservative extra refreshes. The $240$ sequences span $57$ houses, with $55$, $9$, $27$, $54$, $40$, and $55$ sequences for the six families in the order listed. The $11$ instructions are held-out-template phrasings from the request benchmark and are disjoint from the six compiler examples; two also appear among the minimal pairs.

\noindent\textbf{Request benchmark.}
The request benchmark has 216 requests from 36 templates, six per family, and holds out whole
\emph{templates} rather than sentences, since a sentence split measures rewording; the frozen
compiler evaluation uses the 69 requests of the held-out templates. It also includes 144 minimal pairs in which a modifier changes a specified set of stage-specific dependency edges while leaving unrelated edges unchanged.

\noindent\textbf{Unseen-placement benchmark (Adjacent).}
A seventh family tests a dependency placement absent from all six compiler demonstrations.
The benchmark interprets \emph{a chair in a room next to the one I am in} as a chair in a room connected to the robot's current room by one currently passable doorway, so closing that doorway changes qualification although the rooms remain physically adjacent. The semantic dependencies are inventory, position, and topology. The judge receives passability as a state field; the compiler sees only the phrasing and the variable definitions, in which topology is reachability given doors. The benchmark contains $80$ sequences in $80$ houses, six held-out phrasings, three of which name a door or connection explicitly, and all five event types.

\noindent\textbf{Backbones.}
We evaluate Phi-4, Qwen3-8B and Qwen3-14B (\text{microsoft/phi-4},
\text{Qwen/Qwen3-8B}, \text{Qwen/Qwen3-14B}) as locally deployable models through
HuggingFace Transformers, and gpt-oss-120b through a hosted OpenAI-compatible endpoint.
Decoding is greedy (temperature $0$ at the endpoint), Qwen3 thinking mode is disabled, compiler outputs are capped at $512$ new tokens, and judge verdicts at $8$. Local weights load in bfloat16 on one L40S GPU (Transformers 5.8, PyTorch 2.10), and judge prompts run in chunks of four per forward pass. The fallible-evaluator experiments use frozen Phi-4 and Qwen3-14B verdicts; the timing study uses Phi-4.  

\noindent\textbf{Policies.}
Within each evaluator-interface condition, policies share the semantic evaluator, planner, deterministic selector, and composition rule \eqref{eq:compose}. They differ in the maintenance schedule, including semantic refreshes and reselection triggers. Table~\ref{tab:rq12} compares full recomputation, a static inventory-only signature, a universal semantic signature, reference variables with and without stage placement, and the compiled signature. Reference signatures are oracle controls, not inferred baselines. The exact predicate, the affordance sets it scores, and the selector are benchmark components fixed per family: the selector returns the fewest-hop feasible region for nearest-candidate requests, breaking ties on the display name, and the alphabetically first feasible region otherwise. Every policy reruns this selector after every event, so the inventory-only row fails only through stale semantic verdicts and never through suppressed reselection. Maintenance policies receive the compiled signature and the footprint, not the family label.  

\noindent\textbf{Semantic evaluators and prompt interfaces.}
An exact benchmark predicate isolates maintenance errors from semantic-model errors. Fallible-judge
experiments use a \emph{frozen verdict tape}: each distinct regional question is evaluated once per interface, and the resulting verdict is replayed to every policy; an unparsable generation aborts the run. The \emph{reference-field interface} constructs prompts from $D^*_{\mathrm{sem}}(q)$ and holds these fields fixed independently of the predicted signature. A prompt carries exactly one line per declared variable: the candidate region's contents for inventory, and position, history, and adjacency as yes/no relations to that region; the tape key is the full prompt string. The \emph{compiled-field interface} instead uses $\widehat D_{\mathrm{sem}}(q)$ without reference fields. An omitted field can make both policies agree on an incorrect predicate, so task accuracy is evaluated separately from agreement. Phi-4 recovers the semantic slots exactly on the $11$ main dynamic instructions, making the two interfaces coincide there; the unsafe-signature and Adjacent experiments test their divergence. 

\noindent\textbf{Metrics.}
For five-event sequences, \emph{episode accuracy} requires all five post-event targets to match the benchmark reference. \emph{Incremental agreement} measures target agreement with full recomputation under the same evaluator and interface. \emph{Semantic evaluations} count regional predicate evaluations including initialization, as in \eqref{eq:cost}; under a frozen tape these are logical refresh costs, not repeated inference calls. \emph{Conservative containment} is \eqref{eq:cons}, and \emph{semantic inflation} is $|\widehat D_{\mathrm{sem}}|/|D^*_{\mathrm{sem}}|$.

\section{Results}
\label{sec:results}
Every number below is measured on the frozen evaluation units of Section~\ref{sec:benchmarks} with the metrics defined there. Semantic-evaluation counts are reported per five-event sequence. Table~\ref{tab:rq12} reports episode-accuracy proportions with $95\%$ Wilson intervals and semantic-evaluation means with standard deviations. Decision-agreement rates are descriptive proportions over the evaluated decisions; the observation-derived footprint experiment reports a one-sided Clopper--Pearson bound.

\subsection{Ablation: the same targets, fewer rooms rechecked}

\noindent\textbf{Conservative coverage and update scope.}
Table~\ref{tab:rq12} separates target accuracy from semantic cost using the exact predicate. The static inventory-only policy attains $0.479$ episode accuracy and $0.703$ invalidation recall, whereas all containing signatures preserve the $1{,}200$ decisions. Along the ablation sequence, regional footprints reduce mean evaluations from $30.6$ to $23.4$, request-specific variables to $14.1$, and stage separation to $11.2$; these are sequential differences, not order-independent contributions. Stage separation matters for the families with selection dependencies: $18.8\to8.3$ (\emph{nearest-candidate}) and $23.5\to13.2$ (\emph{nearest+history}). 

\begin{table}[h]
\vspace{5pt}
\centering
\caption{Maintenance ablation on 240 five-event robot episodes with the exact benchmark predicate.
Episode accuracy with $95\%$ Wilson intervals; semantic evaluations per sequence, mean $\pm$ SD.}

\vspace{-5pt}
\label{tab:rq12}
\scriptsize\setlength{\tabcolsep}{3pt}
\begin{tabular}{@{}lccc@{}}
\toprule
policy & episode acc. & sem.\ evals & inv.\ recall\\
\midrule
full recomputation             & 1.000 $[.984,1]$ & $30.6\pm9.7$ & 1.000\\
static inventory-only          & 0.479 $[.417,.542]$ & $8.1\pm1.6$ & 0.703\\
universal semantic signature   & 1.000 $[.984,1]$ & $23.4\pm6.4$ & 1.000\\
reference variables, no stage split & 1.000 $[.984,1]$ & $14.1\pm6.6$ & 1.000\\
reference stage-specific       & 1.000 $[.984,1]$ & $11.2\pm3.6$ & 1.000\\
\textbf{MIND, Phi-4 ($K{=}2$)} & 1.000 $[.984,1]$ & $11.2\pm3.6$ & 1.000\\
\bottomrule
\end{tabular}
\vspace{-5pt}
\end{table}

\noindent\textbf{Compile-once signatures outperform event-wise relevance decisions.}
The compiler-free alternative shows the instruction and each event to a language model and asks
whether recomputation is needed (\textsc{none}, \textsc{selection}, or \textsc{semantic}), with
the compiler's six examples and stage definitions in the prompt; everything else is shared.
Compile-once signatures reach $1.000$ episode accuracy at $11.2$ evaluations from 22 compiler
consultations with no non-containing decision. Event-wise relevance reaches $0.658$ at $13.1$
evaluations: of 1200 decisions, 184 omit a required semantic dependency and 165 add a spurious
one. The Phi-4 decider is asked once per event ($1200$ logical calls, $434$ distinct prompts since identical prompts are replayed) and returns different verdicts for $18$ of the $55$ instruction and event-type pairs recurring with other room names.  

\subsection{Instruction-conditioned compilation for new requests}

\noindent\textbf{Containment versus exact recovery.}
On the $69$ held-out requests (Table~\ref{tab:rq3}), Phi-4 and Qwen3-14B produce only conservative errors across six demonstration orders. Qwen3-14B attains containment $1.000\pm0.000$ despite exact recovery $0.454\pm0.097$; Qwen3-8B attains $0.988\pm0.019$ containment. gpt-oss-120b has the highest exact recovery ($0.935$), without inflation or non-containing outputs.

\noindent\textbf{Sample union.}
For Qwen3-8B, increasing $K$ from $1$ to $3$ raises containment from $0.988$ to $1.000$ and removes all non-containing outputs, but lowers exact recovery from $0.372$ to $0.210$ at $21\%$ inflation. Means are over six demonstration orders.

\begin{table}[h]
\centering
\vspace{3pt}

\caption{Dependency-compiler evaluation on the frozen 69-instruction held-out set.}
\vspace{-5pt}
\label{tab:rq3}
\scriptsize\setlength{\tabcolsep}{3.5pt}
\begin{tabular}{@{}lcccc@{}}
\toprule
compiler, $K$ & exact sig. & containment & inflation & non-containing (total)\\
\midrule
gpt-oss-120b, $K{=}1$ & 0.935 & 1.000 & 1.00 & 0\\
Phi-4, $K{=}1$        & 0.727 & 1.000 & 1.04 & 0\\
\textbf{Phi-4, $K{=}2$ (deployed)} & 0.667 & \textbf{1.000} & 1.06 & \textbf{0}\\
Qwen3-14B, $K{=}1$    & 0.454 & 1.000 & 1.08 & 0\\
Qwen3-14B, $K{=}2$    & 0.353 & 1.000 & 1.11 & 0\\
Qwen3-8B, $K{=}1$     & 0.372 & 0.988 & 1.12 & 5\\
Qwen3-8B, $K{=}2$     & 0.273 & 0.998 & 1.17 & 1\\
Qwen3-8B, $K{=}3$     & 0.210 & \textbf{1.000} & 1.21 & \textbf{0}\\
\bottomrule
\end{tabular}
\vspace{-5pt}

\end{table}

On the 144 minimal pairs, scored by whether the variant's signature gains exactly the modifier's
variables in the right slot (edge-correct) and nothing in the other slot (clean), Phi-4 reaches
$0.854$ and $0.799$, gpt-oss-120b $0.743$ and $0.729$, Qwen3-14B $0.681$ and $0.465$, and
Qwen3-8B $0.639$ and $0.569$. The hard case is the modifier that should add position and topology to \emph{selection}: edge-correct for $0$ of $36$ pairs under both Qwen models, which place it in the semantic slot ($0.58$ for Phi-4, $0.47$ for gpt-oss-120b), a conservative promotion matching the no-stage-split ablation.

\noindent\textbf{Conservative repair on Adjacent.}
For Adjacent's passable one-hop rule (Section~\ref{sec:benchmarks}), only gpt-oss-120b places topology in the semantic slot across all six demonstration orders (raw containment $1.000$); Qwen3-14B, Phi-4, and Qwen3-8B place it in selection ($0.722$, $0.556$, $0.444$). Stage widening~\eqref{eq:widen} repairs the demotion to $1.000$ containment for the first two, and Qwen3-8B additionally needs a $K{=}3$ union to recover omitted variables. With the exact benchmark predicate, every repaired signature reaches $1.000$ episode accuracy at
$18.3$ semantic evaluations, compared with $1.000$ at $30.4$ for full recomputation. With a Phi-4 judge over $400$ decisions, the raw Phi-4 signature yields agreement of $0.960$ under the reference-field interface and $0.953$ under the compiled-field interface; widening restores agreement on all $400$ decisions under both interfaces. The $1.66\times$ reduction is smaller than in the primary benchmark because position and topology affect qualification and trigger global refreshes. Full-recomputation episode accuracy remains $0.06$ under the compiled-field interface, whose prompt fields come from the raw, unwidened Phi-4 signature, and $0.00$ under the reference-field interface, with final-step accuracy of $0.38$ and $0.31$, respectively. Widening, adopted after these demotions were observed, reads no labels; on the primary benchmark it keeps every decision at $15.3$ evaluations.

\begin{table}[h]
\centering
\vspace{3pt}
\caption{Interface ablation: a missed dependency and its conservative repair (Phi-4 judge,
$1200$ decisions).}

\label{tab:negctl}
\vspace{-5pt}
\scriptsize\setlength{\tabcolsep}{3pt}
\begin{tabular}{@{}llccc@{}}
\toprule
compiler signature & interface & agreement & episode acc. & evals\\
\midrule
Qwen3-8B, one sample & reference & 0.992 ($1190/1200$) & 0.329 & 10.4\\
Qwen3-8B, one sample & compiled  & 1.000 & 0.329 & 10.4\\
Qwen3-8B, $K{=}2$ union & reference & \textbf{1.000} & \textbf{0.342} & 11.5\\
Phi-4, $K{=}2$ (main) & both & 1.000 & 0.342 & 11.2\\
\bottomrule
\end{tabular}
\vspace{-5pt}
\end{table}

\subsection{Effects of missed dependencies and footprints}

\noindent\textbf{Missed dependencies and repair.}
A Qwen3-8B signature omits semantic history (Table~\ref{tab:negctl}). Under the reference-field interface, stale reuse changes $10/1200$ decisions, reducing agreement to $0.992$ and episode accuracy from $0.342$ to $0.329$, while evaluations fall from $11.2$ to $10.4$. Under the compiled-field interface, both policies lack history and agree on all $1200$ decisions, but both score $0.329$ episode accuracy. The same three history episodes fail through different mechanisms: stale reuse versus an incomplete evaluator input. A $K{=}2$ union restores reference-field agreement on all decisions and accuracy $0.342$ at $11.5$ evaluations. 

\noindent\textbf{Approximate event footprints.}
Perturbing the footprint alone over $240$ sequences with the exact predicate, over-approximation never costs correctness (agreement, accuracy, and recall stay $1.000$ while evaluations rise from $11.2$ to $18.5$ when every untouched region is added); under-approximation degrades monotonically as evaluations fall to $8.2$, its stale-step rate tracking the fraction of partial footprints (Pearson $0.997$). Mislabeling the changed variable behaves the
same way: acting on the detector's top guess falls to $0.683$ accuracy at a $40\%$ error rate,
whereas invalidating over the whole candidate set holds $1.000$ at $27.5$ evaluations versus $30.6$.

\subsection{Closed-loop and observation-derived robot updates}

\noindent\textbf{Inferred maintenance matches full recomputation under fallible evaluators.}
With frozen Phi-4 and Qwen3-14B evaluators, inferred maintenance matches full recomputation on
all \text{1200 post-event decisions across 240 sequences} at $11.2$ semantic evaluations versus
$30.6$. The static inventory-only rule reaches agreement $0.958$ and $0.983$, respectively.
Full-recomputation episode accuracy is $0.342$ with Phi-4 and $0.275$ with Qwen3-14B.
That statistic compounds evaluator error over five decisions (final-step accuracy is $0.704$ and
$0.638$); inferred maintenance reproduces both figures exactly, so agreement measures maintenance
fidelity, not semantic-model competence.

\noindent\textbf{Selective maintenance reduces semantic-evaluator runtime.}
On the separate 59-sequence timing subset, full recomputation averages $29.2$ semantic evaluations
and $1.44\,\mathrm{s}$ of Phi-4 evaluator time per five-event sequence, with $1722$ regional consultations and $354$ batched forward passes in total. Inferred maintenance averages $10.2$ evaluations and $0.58\,\mathrm{s}$, with $602$ consultations and $207$ batches. The measured evaluator-time reduction is $2.49\times$ before compilation. Charging one $0.56\,\mathrm{s}$ instruction compilation to each sequence gives $1.14\,\mathrm{s}$ of accounted model time versus $1.44\,\mathrm{s}$ for full recomputation, a $0.30\,\mathrm{s}$ saving. Savings depend on how many room judgments each event invalidates; compilation is amortized only when accumulated inference-time savings exceed its cost. The wall-clock figures are specific to one local model and one benchmark horizon. Timings are fresh inference with the tape disabled, wall clock around the generation calls on one L40S in bfloat16, one run without warm-up exclusion; the $0.56\,\mathrm{s}$ covers both $K{=}2$ compiler consultations. These figures exclude observation acquisition, change detection, planning, and motion.

\noindent\textbf{Online footprints inferred from simulator state.}
The preceding experiments replay supplied footprints. In the closed-loop AI2-THOR~\cite{kolve2017ai2thor} study, the controller instead obtains footprints by differencing consecutive simulator states. Qualification uses the exact benchmark predicate. Each episode applies five physical events in a fixed order (door closure, relocation, robot move, use, second relocation) and ends after the fifth; agreement is on the final target decision. Reachability is a breadth-first search over the live navigation grid, which a closed door splits, and \emph{nearest} here is navigation-grid path length rather than doorway hops. This evaluates maintenance during execution with metadata-derived changes, not perception-driven grounding. Across $32$ episodes spanning all six families, inferred maintenance matches the full-recomputation target decisions in all $32$ episodes at $10.47$ semantic evaluations versus $26.00$, a $2.48\times$ reduction, with no stale support entries. Event-wise relevance matches the reference on $30/32$ episodes at $11.97$ evaluations.

\begin{figure}[t]
\centering
\includegraphics[width=.86\columnwidth]{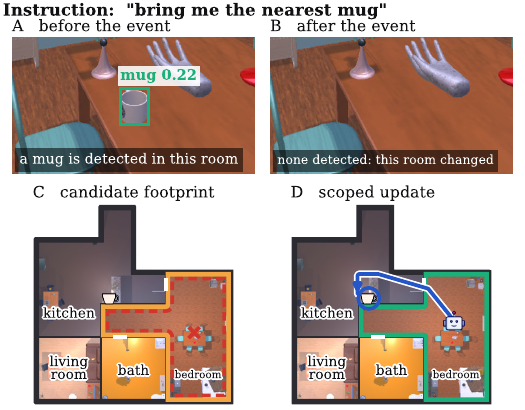}
\vspace{-2pt}
\caption{\textbf{Removal localization} in ProcTHOR validation house 571. (A, B) Matched views before and after mug removal. (C) The candidate footprint contains the affected bedroom, the only room refreshed here. (D) Reselection identifies the kitchen mug; the living room remains unreachable.}
\label{fig:perception}
\vspace{-10pt}
\end{figure}

\noindent\textbf{Observation-derived removal footprints.}
OWLv2~\cite{minderer2023owlv2} (\text{google/owlv2-large-patch14-ensemble} on $896\times672$ renders) processes matched before/after room scans (Fig.~\ref{fig:perception}) for known one-object removals, one standpoint per room at four headings, every room rescanned after the event from the same standpoints. Given the removed object type $T$, the candidate footprint includes rooms whose detection count for $T$ changes and rooms whose prior count reaches the conservative multiplicity threshold. The operating point (detection threshold $0.05$, multiplicity $6$), fixed on $41$ calibration events in $41$ houses, contains the affected room for all $54$ held-out houses, a disjoint index range of the same split, with a mean footprint of $2.19$ of $4.61$ rooms and a one-sided $95\%$ Clopper--Pearson lower bound of $0.946$. This evaluates affected-room localization given matched observations and a supplied event type, not event discovery; scan and detector cost is excluded.
 
\begin{figure*}[t]
\centering
\includegraphics[width=0.98\textwidth]{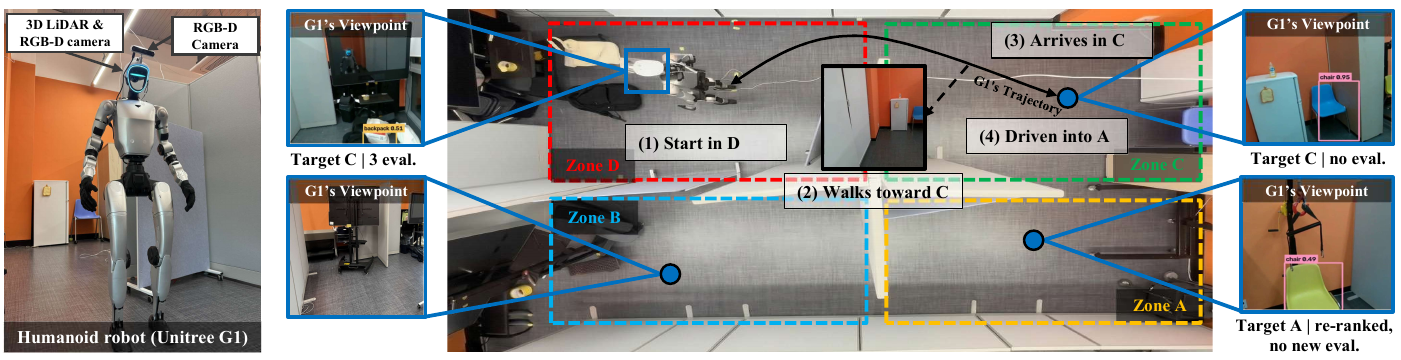}
\vspace{2pt}
\caption{\textbf{Unitree G1 hardware validation.} Left: the robot, with the 3D lidar and the RGB-D camera used for localization and detection. Centre: ceiling view of the four-zone laboratory, chairs in A and C and a backpack in the start zone D, with the path taken in one run of scenario 2, \emph{the nearest thing to sit on}. Insets: the robot's camera in each zone; the three moments of the run carry the target and the evaluation count from the module log. Three initial evaluations select C, one hop away, and the robot walks there and arrives. When the operator then drives it into A, the target re-ranks to A on position alone; position is a selection dependency for this request, so no room is re-judged.}
\label{fig:g1}
\vspace{-10pt}
\end{figure*}

\subsection{Physical-robot demonstration}
Fig.~\ref{fig:g1} shows MIND on a Unitree G1 humanoid in a four-zone laboratory: ROS~2 with lidar localization on a frozen map, Nav2 planning, a YOLO-World detector, and a 4-bit on-board judge (Qwen3-8B in scenarios 1 and 2, Phi-4 in 3 and 4). Signatures came from the deployed Phi-4 compiler, declared to the controller. Provided: zone polygons, candidate regions, and a scoring-only ground-truth inventory; observed: the robot's region, planner reachability, and the detector's inventory. 

Four scenarios were run, $26$ scored runs and $8$ excluded with reasons. (1)~\emph{Something to sit on}, chairs in A and C, a backpack in the start room D: three evaluations selected A, and the walk through C, a transient loss and recovery of the route to A, and the arrival re-judged no room. MIND used $3$ semantic evaluations and approximately $0.33\,\mathrm{s}$ of judge time per run over five runs, compared with $9$ evaluations and approximately $1.0\,\mathrm{s}$ for full recomputation over three. (2)~\emph{The nearest thing to sit on}, same layout and move: the same move re-ranked the target from C to A without additional semantic evaluation under MIND, $3$ against $9$; in both scenarios all scored runs reached the expected target. (3)~\emph{A chair in the room I am in}, chairs in all three rooms: position is semantic, so every move re-judged every room under both policies and nothing was saved, as \eqref{eq:cost} predicts for a global refresh. The judge read the inventory line but not the position line, so every decision in the six runs inherited the same error while navigation and perception passed; the two policies agreed on every decision, as a fixed evaluator requires. (4)~\emph{Something to sit on}, then a barrier closing the route to A and the chair removed from C: the barrier refreshed reachability only and the removal re-judged one room, $4$ evaluations against $15$ to $18$, and both policies abstained; one floor span was declared as wall so the barrier closed the only route. Evaluation counts per event were identical across repetitions; one full-recomputation run was excluded for a detector miss, a perception failure scored apart from the decision. The physical-robot demonstrations are also shown in the supplementary video.

\section{Conclusion}
MIND combines instruction-conditioned qualification and selection dependencies with event footprints to reduce semantic re-evaluation. Under Proposition~\ref{prop:equiv}'s assumptions, selective maintenance preserves full-recomputation target decisions. The evaluations separate savings from spatial scope, request-specific variables, and stage placement, and expose the consequences of missed dependencies and footprints.

The controlled benchmarks assume populated semantic maps and a restricted state vocabulary; dependency labels are not independently annotated. Metadata-derived closed-loop updates and known-removal image localization test different components. Frozen verdict tapes isolate maintenance fidelity, not task correctness or online determinism. The hardware evidence is four scenarios on one layout. In the current-region case both policies inherited the same verdicts from an evaluator that misread the position field; semantic judgment is confined to one replaceable evaluator rather than special-cased in code, so a stronger judge changes the maintenance rule by nothing. Broader instruction coverage and robustness under changing evaluator outputs remain open.

\section*{Acknowledgment}
Claude was used to assist with development and debugging of experimental software and data processing. All AI-assisted content was reviewed and validated by the authors.

This research was supported by the NVIDIA Academic Grant Program using NVIDIA RTX PRO 6000 Blackwell GPUs and by the Ministry of Education of the Republic of Korea and the National Research Foundation of Korea (NRF-2026S1A5A8012989)..

\bibliographystyle{IEEEtran}
\setlength{\bibsep}{0pt}
{\scriptsize\bibliography{references}}

\end{document}